\documentclass[twocolumn,twoside,10pt]{IEEEtran}

\usepackage[dvipsnames]{xcolor}
\usepackage{pdflscape}

\usepackage{colortbl}

\usepackage[latin1]{inputenc}%%%
\usepackage{amssymb}
\usepackage{amsmath}%,amsxtra}
\usepackage{dsfont}

\newcommand{\ie}{\emph{i.e.}}
\newcommand{\eg}{\emph{e.g.}}
\newcommand{\cf}{{cf.\;}}

\newtheorem{theorem}{Theorem}%[section]

\newtheorem{definition}[theorem]{Definition}

\newcommand{\argmin}[1]{\mathop{\arg\!\min}_{#1}}

\usepackage{accentbx}%accents}
\newcommand{\dic}[1]{\textrm{\upaccent{\aboxshift{\char"12}}{$#1$}}}

\newcommand{\calC}{\cp{C}}
\newcommand{\calR}{\cp{R}}
\newcommand{\ps}[2]{\langle{#1},{#2}\rangle}
\newcommand{\psH}[2]{\langle{#1},{#2}\rangle_{\H}}
\newcommand{\norm}[1]{\|{#1}\|}
\newcommand{\normH}[1]{\|{#1}\|_\H}

\newcommand{\normHBig}[1]{\Big\|{#1}\Big\|_\H}

\graphicspath{{matlab/}}

\def\R{\ensuremath{\mathds{R}}}

\def\X{\ensuremath{\mathds{X}}} % input space
\def\Y{\ensuremath{\mathds{Y}}} % output space
\def\H{\ensuremath{\mathds{H}}} %Hilbert
\def\D{\ensuremath{\mathcal{D}}} %dictionary
\newcommand{\cb}[1]{{\ifmmode {\boldsymbol{#1}}\else ${\boldsymbol{#1}}$\fi}}
\newcommand{\cp}[1]{\ifmmode {\mathcal{#1}}\else ${\mathcal{#1}}$\fi}

\newcommand{\dx}{\dic{\bx}}
\newcommand{\dkappa}{\dic{\bkappa}}
\newcommand{\dK}{\dic{\bK}}
\newcommand{\dH}{\dic{\H}}

\newcommand{\bx}{\cb{x}}
\newcommand{\bu}{\cb{u}}
\newcommand{\bv}{\cb{v}}

\newcommand{\by}{\cb{y}}
\newcommand{\bxi}{\cb{\xi}}
\newcommand{\balpha}{\cb{\alpha}}
\newcommand{\bA}{\cb{A}}

\newcommand{\bK}{{\cb{K}}}%_{\!\!\not\,\mathrm{c}}}}
\newcommand{\bI}{{\bf I}}

\newcommand{\bkappa}{\cb{\kappa}}

\begin{document}

%\title{Centering data %Bridging the gap between centered and uncentered data 
%in machine learning: \\an eigenanalysis}% of the eigendecomposition problem}%:  between the eigenpairs of the Gram matrices}

\title{%On the equivalence of functional and parameter adaptations for online learning with kernels}
Analyzing sparse dictionaries \\for online learning with kernels}% : when dictionary learning meets feature selection}
%\title{An eigenanalysis of the impact of \\centering data in machine learning}

%\author{Paul Honeine%,~\IEEEmembership{Member,~IEEE}
%\IEEEcompsocitemizethanks{\IEEEcompsocthanksitem 
%\thanks{M. Honeine is with the Institut Charles Delaunay (CNRS), Université de Technologie de Troyes, Troyes, France.}
%}%
\author{Paul~Honeine,~\IEEEmembership{Member~IEEE}
\thanks{P.~Honeine is with the Institut Charles Delaunay (CNRS), Universit\'e de technologie de Troyes, 10000, Troyes, France. Phone: +33(0)325715625; Fax: +33(0)325715699; E-mail: paul.honeine@utt.fr
}}

\markboth{}%IEEE Trans. on Signal Processing,~Vol.~XX, No.~XX,~XX~201X}
{Honeine: Analyzing Sparse Dictionaries}% for Online Learning With Kernels}

%\begin{document}

%\editor{~\\
%~\\
%~\\
%\mbox{~}\hfill ``All happy families are alike; each unhappy 
%\\\mbox{~}\hfill  family is unhappy in its own way.''
%\\\mbox{~}\hfill  Leo Tolstoy}%: Anna Karenina (1878)}

\sloppy

\maketitle

%\IEEEcompsoctitleabstractindextext{
\begin{abstract}
Many signal processing and machine learning methods share essentially the same linear-in-the-parameter model, with as many parameters as available samples as in kernel-based machines. Sparse approximation is essential in many disciplines, with new challenges emerging in online learning with kernels. To this end, several sparsity measures have been proposed in the literature to quantify sparse dictionaries and constructing relevant ones, the most prolific ones being the distance, the approximation, the coherence and the Babel measures.  In this paper, we analyze sparse dictionaries based on these measures. By conducting an eigenvalue analysis, we show that these sparsity measures share many properties, including the linear independence condition and inducing a well-posed optimization problem. Furthermore, we prove that there exists a quasi-isometry between the parameter (i.e., dual) space and the dictionary's induced feature space.

\end{abstract}


%
%\makeatletter
%\if@draftclsmode
%\begin{center}
%\bfseries EDICS Category: SSP-SSEP
%\end{center}
%\fi
%\makeatother
\IEEEpeerreviewmaketitle

%%\section*{Table of contents}
%\newpage
%
%\tableofcontents
%
%\newpage

\section{Introduction}

\PARstart{S}{parse} approximation is essential in many disciplines due to the advent of data deluge in the era of ``Big Data'', as illustrated by the extensive literature of compressed sensing (see \cite{spmag-cs} and references therein). Sparsity promoting is crucial in signal processing and machine learning, such as Gaussian processes \cite{Csato01}, kernel-based methods \cite{Wu06}, Bayesian learning \cite{Wipf03perspectives}, as well as neural networks \cite{Platt91} with pruning \cite{Cun90OptimalBrainDamage} and the more recent dropout principle in deep learning \cite{Dropout14}. 

Many learning machines share essentially the same model, in a linear or a nonlinear --- kernel --- form, including support vector machines \cite{Vap98}, Gaussian processes \cite{gpml} and radial-basis-function networks such as resource-allocating networks \cite{Platt91} and more recently neural networks for function approximation \cite{Huang05ageneralized}; see also the seminal work of Poggio and Smale \cite{Poggio03themathematics}. All these learning machines rely on the well-known ``Representer Theorem'' \cite{Representer}, which defines a linear-in-the-parameters model with as many parameters as training samples.

A sparse approximation of this model is often required for many interesting and desirable properties, such as enforcing the interpretation of the results and providing a computational tractable problem for large-scale datasets. While this issue has been studied within the last 15 years in kernel-based machines \cite{Scholkopf99inputspace,ZeroNorm03}, recent developments in sparse approximation and compressed sensing open the way to new advances. Moreover, online learning brings new challenges to sparsity in signal processing and machine learning, when a new sample is available at each instant, thus leads to an incrementation of the number of parameters. Therefore, one needs to control such complexity growth, by selecting samples that take part in the model formulation; in the literature, these contributing samples are called atoms and are collected in a set called dictionary.

The construction from available samples of a pertinent dictionary and the measure of its relevance have been investigated in the literature with several sparsification criteria, each being coupled with a sparsity measure that defines the diversity captured by the dictionary. %.been investigated in the literature, depending on the definition of relevance, \ie, with diversity quantified with
%In order to construct a relevant dictionary from available samples, several sparsification criteria have been investigated in the literature, depending on the definition of relevance, \ie, with diversity quantified with a sparsity measure. 
The oldest sparsity criterion is the distance introduced in \cite{Platt91} for controlling the complexity of the structure of radial-basis-function networks in resource-allocating networks \cite{Rosipal97RAN}; see also \cite{%Babu2013,
Yang2013,Vukovic2013} for recent advances. The criterion constructs a dictionary by lower-bounding the pairwise distance between its atoms. Another criterion, the approximation criterion, explores a more deeper analysis of the atoms, by lower-bounding the error of approximating any atom by the other atoms, as investigated in \cite{Csato02} for Gaussian processes, in \cite{Eng04} for a kernel recursive least squares algorithm, %in \cite{Pokharel2009} for a kernel least mean square algorithm, 
and more recently in \cite{12.tpami} for a kernel principal component analysis. A third criterion takes advantage of recent developments in the sparse approximation literature \cite{Tro04} and compressed sensing \cite{Elad2010book}, by upper-bounding the coherence between any pair of atoms. Initially introduced for online learning with kernels \cite{Hon07.isit,Ric09.tsp} and learning in sensor networks \cite{Hon08.globecom,Hon10.eur}, it has been extensively considered for one-class classification \cite{12.eusipco.oneclass,12.ssp.one_class}, for online learning with multiple kernels \cite{Yukawa12Multikernel,Tobar14Multikernel} and multiple dictionaries \cite{Ishida2013Multikernel} and for multiple-output learning \cite{13.spl.dictionary}. The Babel measure and its criterion provide a more comprehensive analysis of the dictionary structure, by limiting the cumulative coherence \cite{Fan2014}. To the best of our knowledge, there is no work that studies all these sparsity measures and criteria.

Independently of the sparsification criterion and the resulting dictionary, many algorithms have been introduced to update the model. As it might be expected, the wide class of linear adaptive filters has been extensively investigated for online learning with kernels, by revisiting popular algorithms such as the least mean squares (LMS), the normalized LMS (NLMS), the affine projection (AP), and the recursive least squares (RLS) algorithms; see for instance \cite{Say03} for a review of linear adaptive filters. There exists two frameworks to develop adaptive algorithms in online learning with kernels, thanks to the underlying linear-in-the-parameters model: a functional (\ie, feature) framework and a dual (\ie, parameter) one. Within the functional framework, the optimization is operated in the feature space, by estimating and updating within the subspace spanned by the atoms of the dictionary. This framework has been widely investigated for online learning with kernels; see for instance \cite{Kiv04,Yukawa12Multikernel} 
%\cite{Yukawa12efficient}, 
as well as \cite{Sma06} for a theoretical analysis and \cite{KernelAdaptiveFiltering} for a comprehensive study. The second framework is based on estimating the parameters of the model, thus solving an optimization problem in the so-called dual space. %This space is the vector space $\R^m$, $m$ being the model's order (\ie, the cardinality of the dictionary). 
This framework has been extensively explored in the literature due to its simplicity, with a NLMS algorithm \cite{Hon07.isit}, an AP algorithm \cite{Ric09.tsp}, and a RLS algorithm \cite{Eng04,Hon07.gretsi.b}. For an overview of this framework, see \cite{07.PhD} and references therein. To the best of our knowledge, only Yukawa pointed out the distinction between these two frameworks in \cite[Section 6.6.4]{Yukawa_lecture}. The relationship between the two frameworks has not been studied before, namely connecting the feature space to the dual space.%; see also \cite[Section~2]{Yukawa12efficient} for a vector space projection viewpoint.

The aim of this paper is to study all the aforementioned sparsity measures and sparsification criteria (\cf Section~\ref{sec:intro_spars}). To this end, we provide an analysis of the eigenvalues associated to a sparse dictionary, and provide upper and lower bounds in terms of the sparsity measures (\cf Section~\ref{sec:bounds}). %These bounds are the cornerstone of this work. 
We show that the lower bounds provide conditions on the linear independence of the atoms (\cf Section~\ref{sec:lin.indep}). Moreover, we show that the condition number of the Gram matrix associated to a sparse dictionary is upper-bounded, illustrating the impact of the sparsity measures on the conditioning of the optimization problem (\cf Section~\ref{sec:cond.number}). A major result provided in this paper is the connection between the dictionary's induced feature space and the dual space, by showing that there exists a quasi-isometry between these spaces when dealing with sparse dictionaries%, and quantifying the impact of the sparsity measure on this topology
. These results allow to bridge the gaps between the two aforementioned frameworks (\cf Sections~\ref{sec:isometry} and \ref{sec:isometry_ip}). The big picture is illustrated in \tablename~\ref{tab:birdseye}.

%The remained of this technical report is organized as follows. Next section introduces the kernel-based machines for online learning. Section~\ref{sec:filter} presents many filter-based sparsification criteria, while Section~\ref{sec:topology} provide preliminary analysis with the study of the dictionary finiteness and give some fundamental results. Section~\ref{sec:eigen} describes an eigenanalysis on the corresponding Gram matrix, by constraining its eigenvalues and studying its condition number as well as the linear dependence of its atoms. Section~\ref{sec:approx.error} investigates the approximation errors based on the resulting dictionary, such as the error of discarding sample and the error of approximating any feature including the most relevant principal axes. Several connections with the literature of compressed sensing are described in Section~\ref{sec:compressed_sensing}. Section~\ref{sec:final_remarks} concludes this document.

\section{Kernel-based learning machines}

A learning problem aims to find the relation $\psi(\cdot)$ between a compact subspace of a Banach space $\X$ of $\R^d$ and a compact $\Y$ of $\R$ called output space, from on %. Unfortunately, the probability distribution that governs all pairs $(\bx,y) \in \X \times \Y$ is unknown, but is only available from 
a set of available samples, denoted $\{(\bx_1,y_1), (\bx_2,y_2), \ldots , (\bx_n,y_n)\}$ with $(\bx_k,y_k) \in \X\times\Y$. 

\subsection{Batch learning with kernels}

Considering a given loss function $\calC(\cdot,\cdot)$ defined on $\Y \times \Y$ that measures the error between the desired output and the estimated one with $\psi(\cdot)$, the optimization problem consists in minimizing a regularized empirical risk as follows
\begin{equation}\label{eq:risk}
    \argmin{\psi(\cdot) \in \H} %\frac{1}{n} 
    \sum_{i=1}^n \calC(\psi(\bx_i),y_i) + \epsilon \, \calR (\|\psi(\cdot)\|_\H^2),
\end{equation}
where $\H$ is the space of candidate functions and $\epsilon$ controls the tradeoff between the fitness error (first term) and the regularity of the solution (second term) where $\calR(\cdot)$ is a monotonically increasing function. Examples of loss functions are the quadratic loss $| \psi(\bx_i) - y_i |^2$ and the hinge loss $( 1 - \psi(\bx_i) y_i )_+$ of the support vector machines.
%, and the logistic regression $\log( 1 + \exp(-\psi(\bx_i) y_i))$. 
%Moreover, the use of the unsupervised loss function $- |\psi(\bx_i)|^2$ is related to the principal component analysis (see for instance \cite{12.tpami}).

By using the formalism of the reproducing kernel Hilbert space (RKHS) as the space $\H$ of candidate functions, kernel-based machines incorporate prior knowledge by using a kernel. Let $\kappa\!:\X \times \X\rightarrow\R$ be a reproducing kernel, and $(\H,\psH{\cdot}{\cdot})$ the induced RKHS with its inner product. The reproducing property states that any function $\psi(\cdot)$ of $\H$ can be evaluated at any sample $\bx_i$ of $\X$ using $\psi(\bx_i) = \psH{\psi(\cdot)}{\kappa(\bx_i,\cdot)}$. 
This property shows that any sample $\bx_i$ of $\X$ is represented with $\kappa(\cdot,\bx_i)$ in the space $\H$, also called feature space. Moreover, the reproducing property leads to the so-called kernel trick, that is for any pair of samples $(\bx_i,\bx_j)$, we have $\psH{\kappa(\cdot,\bx_i)}{\kappa(\cdot,\bx_j)} = \kappa(\bx_i,\bx_j)$. %In particular, $\normH{\kappa(\cdot,\bx_i)}^2 = \psH{\kappa(\cdot,\bx_i)}{\kappa(\cdot,\bx_i)} = \kappa(\bx_i,\bx_i)$. 
Commonly used kernels are the linear kernel with $\ps{\bx_i}{\bx_j}$, the polynomial kernel $\left(\ps{\bx_i}{\bx_j} + c\right)^p$ and the Gaussian kernel $\exp\left(\frac{-1}{2\sigma^2} \|\bx_i - \bx_j\|^2\right)$.

% From these kernels, only the Gaussian kernel is unit-norm, that is $\normH{\kappa(\bx,\cdot)}=1$ for any sample $\bx \in \X$. In this paper, we do not restrict ourselves to a particular kernel. We denote 
%\begin{equation*}
%	\boxed{~r^2 = \inf_{\bx \in \X} \kappa(\bx,\bx)~}
%		\qquad \text{and} \qquad 
%	\boxed{~R^2 = \sup_{\bx \in \X} \kappa(\bx,\bx)~}
%\end{equation*}
%where $\kappa(\bx,\bx) = \normH{\kappa(\bx,\cdot)}$. For unit-norm kernels, we get $R=r=1$.

The Representer Theorem is a cornerstone of kernel-based machines \cite{Representer}. It states that the solution of the optimization problem \eqref{eq:risk} takes the form
\begin{equation}\label{eq:repr}
	\psi(\cdot) = \sum_{i=1}^{n} \alpha_i \, \kappa(\bx_i,\cdot).
\end{equation}
%This result is studied by \, and a sketch of proof of this theorem is given in the footnote\footnote{To prove the Representer Theorem \eqref{eq:repr}, we decompose any function $\psi(\cdot)$ of $\H$ into $\psi(\cdot)=\sum_{i=1}^n \alpha_i\,\kappa(\bx_i,\cdot) + \psi^\perp(\cdot)$, where $\langle\psi^\perp(\cdot),\kappa(\bx_i,\cdot)\rangle_{\H}=0$ for all $i=1,2,\ldots,n$. On the one hand, any evaluation $\psi(\bx_i)$ is independent of $\psi^\perp(\cdot)$ since $\psi(\bx_i)=\psH{\psi(\cdot)}{\kappa(\bx_i,\cdot)}$. On the other hand, the monotonically increasing function $\calR(\cdot)$ guarantees that $\calR(\|\psi(\cdot)\|_\H^2) = \calR(\|\sum_{i=1}^n \alpha_i\,\kappa(\bx_i,\cdot)+\psi^\perp(\cdot)\|_\H^2) \geq \calR(\|\sum_{i=1}^n \alpha_i\,\kappa(\bx_i,\cdot)\|_\H^2)$, where the Pythagorean theorem is used. Therefore, a null $\psi^\perp(\cdot)$ minimizes the regularization term without affecting the fitness term.}. 
This theorem shows that the functional optimization problem \eqref{eq:risk} is equivalent to the estimation of $n$ unknowns, $\alpha_1, \alpha_2, \ldots, \alpha_n$ in \eqref{eq:repr}. By injecting the above expression into \eqref{eq:repr}, we get the the (often called) dual problem. This duality is illustrated next for the kernel ridge regression problem.

%\begin{landscape}
\begin{table}[t]
%\hspace*{-.75cm}
%\footnotesize
%\scriptsize
\renewcommand{\arraystretch}{1.1} 
%\begin{center}
\begin{tabular}{lccccc}
% & \multicolumn{4}{c}{sparsity measures and sparsification criteria} \\ \cline{2-5}
 					& \rotatebox{90}{Distance} & \rotatebox{90}{Approximation} & \rotatebox{90}{Coherence} & \rotatebox{90}{Babel}
					&%\rotatebox{90}{\cf Section}% 
					\!\!Section\!\!\!\! 
					\\\hline
\rowcolor[gray]{.97} Reference: most known work & \cite{Platt91} & \cite{Csato02} & \cite{Ric09.tsp}%,Khandan2013} 
& \cite{Tro04}  & \ref{sec:intro_spars} \\
Reference: more recent work & \cite{KernelAdaptiveFiltering} & \cite{12.tpami} & \cite{13.spl.dictionary} & \cite{Fan2014} & \ref{sec:intro_spars} \\
\rowcolor[gray]{.97} Eigenvalues: lower bounds & $\checkmark$ & \cite{12.tpami} & \cite{Hon07.isit} & $\checkmark$ & \ref{sec:bounds} \\
Eigenvalues: upper bounds & $\checkmark$ & \cite{12.tpami} & $\checkmark$ & $\checkmark$ & \ref{sec:bounds} \\
\rowcolor[gray]{.97} Linear independence & $\checkmark$ & $\checkmark$ & \cite{Hon07.isit} & \cite{Ric09.tsp} & \ref{sec:lin.indep}\\
Condition number & $\checkmark$ & $\checkmark$ & $\checkmark$ & $\checkmark$ & \ref{sec:cond.number} \\
\rowcolor[gray]{.97} Isometry property: distances & $\checkmark$ & $\checkmark$ & $\checkmark$ & $\checkmark$ & \ref{sec:isometry} \\
Isometry property: inner products & $\checkmark$ & $\checkmark$ & $\checkmark$ & $\checkmark$ & \ref{sec:isometry_ip} \\
\hline
\end{tabular} 
\caption{A birds eye view of the theoretical insights studied in this paper. Some of these results were previously derived for unit-norm atoms, as shown with the references given in the table. In this work, we provide an extensive study that completes the analysis to all sparsity measures. % and sparsification criteria
We derive new theoretical insights on connecting the dual space with the dictionary's induced feature space. All the results are generalized to %, we generalize these results to 
any type of kernel, beyond the unit-norm case.}
%\end{center}
\label{tab:birdseye}
\end{table}
%\end{landscape}

\subsection{Kernel ridge regression algorithms}

In the kernel ridge regression%\footnote{Other formulations exists for the kernel ridge regression, as suggested in \cite[Chapter~2]{Shawetaylor_Cristianini} and \cite{Has09}.}
,  the quadratic loss and regularization are used in the optimization problem, namely
\begin{equation}\label{eq:risk_ridge}
    \argmin{\psi(\cdot) \in \H}
    \tfrac12 \sum_{i=1}^n | \psi(\bx_i) - y_i |^2 + \epsilon \, \tfrac12 \|\psi(\cdot)\|_\H^2.
\end{equation}
By injecting the the model \eqref{eq:repr} in the above expression, we get the following dual optimization problem:
\begin{equation}\label{eq:risk_ridge_dual}
    \argmin{\balpha \in \R^n} 
    \tfrac12 \|  \bK\balpha - \by \|^2 + \epsilon \, \tfrac12 \balpha\!^\top \!\bK \balpha,
\end{equation}
where $\bK$ is the Gram matrix whose $(i,j)$-th entry is $\kappa(\bx_i,\bx_j)$, $\by$ and $\balpha$ are vectors whose $i$-th entries are $y_i$ and $\alpha_i$, respectively. In the above expression, we have used the relation 
\begin{equation*}
\|\psi(\cdot)\|_\H^2 
= \big\| \sum_{i=1}^{n} \alpha_i \kappa(\bx_i,\cdot) \big\|_\H^2 
%= \big\langle \sum_{i=1}^{n} \alpha_i \, \kappa(\bx_i,\cdot), \sum_{i=1}^{n} \alpha_i \, \kappa(\bx_i,\cdot) \big\rangle_\H 
= \!\! \sum_{i,j=1}^{n} \! \alpha_i \alpha_j \kappa(\bx_i,\bx_j)
= \balpha\!^\top \! \bK \balpha.
\end{equation*}
The solution of this optimization problem is given by the ``normal equations'', $( \bK\!^\top \bK + \epsilon \bK\!^\top) \, \balpha = \bK\!^\top \by$, which yields\footnote{The expression \eqref{eq:ridge_solution1} is often simplified to $\balpha = ( \bK + \epsilon \, \bI)^{-1}\by$. This equivalence is granted only when the matrix $\bK$ is nonsingular, an assumption that is unfortunately not satisfied in general. This is due to the linear dependence of the training samples.% as illustrated in the linear case, since if (and only if) there exists some weight vector $\bv$ such that $\bX \bv = \bzero$, then $\bK \bv = \bX\!^\top \! \bX \bv = \bX\!^\top \bzero = 0 \bv$ making $0$ an eigenvalue of the matrix $\bK$. This undesirable property extends to most kernels, such as the polynomial kernel, and is more stressed out when dealing with large-scale datasets.
}
\begin{equation}\label{eq:ridge_solution1}
	 \balpha = \big( \bK\!^\top \bK + \epsilon \bK\!^\top \big)^{-1}\bK\!^\top \by.
\end{equation}

% To overcome this drawback, it is necessary to restrict the optimization problem to a relevant subspace of solutions, as demonstrated throughout this paper with sparsification criteria for online learning. 
 
\subsection*{Regularization: $\|\psi(\cdot)\|_\H$ versus $\|\balpha\|$}

The regularization in the dual optimization problem \eqref{eq:risk_ridge_dual} is essentially a Tikhonov regularization of the form $\|\cb{\Gamma} \balpha\|^2$ (where we have in our case $\cb{\Gamma}\!^\top \cb{\Gamma} = \epsilon \bK$). In the literature, the Tikhonov matrix $\cb{\Gamma}$ is often chosen as the identity matrix, up to a multiplicative constant, giving preference to solutions with smaller norms. The kernel ridge regression becomes 
\begin{equation}\label{eq:risk_ridge_dual_hybrid}
    \argmin{\balpha \in \R^n} %\frac{1}{n} 
    \tfrac12 \|  \bK\balpha - \by \|^2 + \epsilon \, \tfrac12 \|\balpha\|^2.
\end{equation}
With the ``normal equations'' $( \bK\!^\top \bK + \epsilon \, \bI ) \, \balpha = \bK\!^\top \by$, we get% which yields the solution%by the linear system 
\begin{equation*}
	\balpha = \big( \bK\!^\top \bK + \epsilon \, \bI \big)^{-1} \bK\!^\top \by.
\end{equation*}

%and its solution \eqref{eq:ridge_solution1} are essentially 

%There also exists another formulation where one considers the fitness problem in \eqref{eq:risk_ridge} with the ridge-like norm regularization $\balpha$, in the same spirit as in \eqref{eq:risk_ridge_linear}. Therefore, as opposed to the regularization in the functional space $\H$ in \eqref{eq:risk_ridge} yielding $\epsilon \, \balpha\!^\top \bK \balpha$ as a regularization term, the regularization in the parameter space $\R^n$ allows to write the optimization problem

Connections between the regularization in the functional space with $\|\psi(\cdot)\|_\H$ and the regularization in the dual space with $\|\balpha\|$ are not straightforward. The only result is based on the fact that $\|\psi(\cdot)\|_\H^2 = \balpha\!^\top \bK \balpha$, and therefore we have from the Rayleigh's quotient and the Courant-Fischer Minimax Theorem \cite[Theorem~8.1.2]{matrix13}:
\begin{equation*}
	\lambda_{\min} \leq \frac{\|\psi(\cdot)\|_\H^2%\balpha\!^\top \bK \balpha
	}{\|\balpha\|^2} \leq \lambda_{\max}
\end{equation*}
where $\lambda_{\min}$ and $\lambda_{\max}$ are the smallest and largest eigenvalues of the Gram matrix $\bK$. As a consequence, minimizing $\|\psi(\cdot)\|_\H^2$ yields the upper bound on the norm of the parameter vector with $\|\balpha\|^2 \leq \lambda_{\min}^{-1} \|\psi(\cdot)\|_\H^2$, while minimizing $\|\balpha\|^2$  yields the following upper bound on the norm in the functional space with $\|\psi(\cdot)\|_\H^2 \leq \lambda_{\max} \|\balpha\|^2$.

It turns out that sparse dictionaries provide models with tighter bounds%between norms in both feature and dual spaces, 
, as studied in detail in Section~\ref{sec:topology}.

% For this purpose, we show that minimizing the former yields an upper bound on the latter; and the vice versa. To this end, 

%we have from  the following pair of inequalities for any vector $\balpha$ of appropriate size:

\subsection{Online learning with kernels}

The Representer Theorem with its linear-in-the-parameters model \eqref{eq:repr} constitutes a bottleneck for online learning, which is required for real-time system identification, Big-Data processing and distributed optimization (\eg, sensor networks). Indeed, in an online setting, the solution should be updated recursively based on a new information available at each instant, namely a novel $(\bx_t,y_t)$ at instant $t$. Thus, by including the new pair $(\bx_t,y_t)$ in the training set, the Representer Theorem dictates a new parameter $\alpha_t$ to be added to the set of unknowns. As a consequence, the order of the linear-in-the-parameters model is continuously increasing.

To overcome this drawback, one needs to control the growth of the model order at each instant, by keeping only a fraction of the kernel functions in the expansion \eqref{eq:repr}. The reduced-order model at instant $t$ takes the form
\begin{equation}\label{eq:repr_m}
       %\boxed{~ 
       \psi_t(\cdot) = \sum_{j=1}^{m} \alpha_{j,t} \, \kappa(\dx_j,\cdot),
       % ~}
\end{equation}
for some order $m$, fixed or controlled, with $m \ll t$. Each $\dx_j$ is chosen 
%The set $\{\dx_1, \dx_2, \ldots, \dx_m\}$ is selected 
from all available samples up to instant $t$, namely\footnote{We consider that each $\dx_j$ is a sample selected from available samples, that is $\dx_j$ is some $\bx_{\omega_j}$ with $\omega_j \in \{1, 2, \ldots, t\}$. By using the notation $\dx_j$ in this paper, as opposed to $\bx_{\omega_j}$, the elements $\dx_j$ in the expansion \eqref{eq:repr_m} need not be samples drawn from the distribution. This difference is investigated in\cite{12.ssp.dictionary,13.gretsi.dictionary}, by updating $\dx_j$ at each instant in order to minimize the prediction error.% This is beyond the scope of this paper.
} $\{\dx_1, \dx_2, \ldots, \dx_m\} \subset \{\bx_1, \bx_2, \ldots, \bx_t\}$. We denote by dictionary the set $\D=\{\kappa(\dx_1,\cdot), \kappa(\dx_2,\cdot), \ldots, \kappa(\dx_m,\cdot)\}$, by atoms its elements, and by 
$\dH$ the space spanned by $\D$. In this paper, we do not restrictive ourselves to unit-norm\footnote{Throughout this paper, we outline the special case of unit-norm atoms since such setting is often considered in the literature. Unit-norm atoms arise when dealing either with the linear kernel when $\|\bx\|=1$ for any $\bx\in\X$, or with a unit-norm kernel, namely $\kappa(\bx,\bx)=1$ for any $\bx\in\X$%, since $\normH{\kappa(\bx,\cdot)}^2= \kappa(\bx,\bx)$
.} atoms. Let
\begin{equation*}
	r^2 = \inf_{\bx \in \X} \kappa(\bx,\bx)
	\qquad \text{and} \qquad
	R^2 = \sup_{\bx \in \X} \kappa(\bx,\bx).
\end{equation*}

 %This allows us to revisit the literature of sparse linear approximation, which corresponds to the use of the linear kernel in our formalism. %Moreover, inspired by the neural network formulation with radial-basis-functions, the elements $\dx_1, \dx_2, \ldots, \dx_m$ are called centers; they turn out to be the atoms when the linear kernel is used.

The optimization problem is two-fold at each instant: selecting the proper dictionary $\D=\{\kappa(\dx_1,\cdot), \kappa(\dx_2,\cdot), \ldots, \kappa(\dx_m,\cdot)\}$ and estimating the corresponding parameters $\alpha_1, \alpha_2, \ldots, \alpha_m$. Before studying in detail the former in Section~\ref{sec:intro_spars}, the latter is outlined next.% in Section~\ref{sec:intro_param}.

\subsection*{Notation}

Throughout this paper, all quantities associated to the dictionary have an accent (by analogy to phonetics, where stress accents are associated to prominence). This is the case for instance of the $m$-by-$1$ vector $\dkappa(\cdot)$ whose $j$-th entry is $\kappa(\dx_j,\cdot)$ and the Gram matrix $\dK$ of size $m$-by-$m$ whose $(i,j)$-th entry is $\kappa(\dx_i,\dx_j)$. The eigenvalues of this matrix are denoted $\dic{\lambda}_{1}, \dic{\lambda}_{2}, \ldots, \dic{\lambda}_{m}$, given in non-increasing order.

%The vector $\dkappa_{\!_{\setminus\!\{\!i\!\}}\!}(\dx_i)$ is the $(m-1)$-by-$1$ column vector with entries $\kappa(\dx_i,\dx_j)$, for $j=1, 2, \ldots m$ and $j \neq i$, and the matrix $\dK_{\!_{\setminus\!\{\!i\!\}}\!}$ is the $(m-1)$-by-$(m-1)$ submatrix of the Gram matrix obtained by removing the $i$-th row and $i$-th column from it, \ie, the entries in $\dK$ associated to $\dx_i$. Its eigenvalues are denoted $\dic{\lambda}_{_{\setminus\!\{\!j\!\}}\!1}, \dic{\lambda}_{_{\setminus\!\{\!j\!\}}\!2}, \ldots, \dic{\lambda}_{_{\setminus\!\{\!j\!\}}\!m-1}$.

%et aussi Sparse Gaussian Process Regression :
%2862-a-matching-pursuit-approach-to-sparse-gaussian-process-regression.pdf
%notamment Section "3 Selection of basis functions"

%$m$ is the size (\ie, cardinality) of the dictionary $\D$.

\subsection{Parameter estimation for online learning}\label{sec:intro_param}

Before studying in Section~\ref{sec:intro_spars} the dictionary in terms of sparsity measures and sparsification criteria for constructing a relevant dictionary, we assume for now that the dictionary is known. From \eqref{eq:repr_m}, the problem of determining the model can be solved in two ways: the functional framework where $\psi_t(\cdot)$ is updated from $\psi_{t-1}(\cdot)$, and the dual framework with the update of the parameter vector $\balpha_t$ from $\balpha_{t-1}$. These two frameworks are summarized next, starting with the latter since its vector-based formulation is straightforward.% due to connections with adaptive filtering literature.

% The first framework, called functional framework, considers a functional optimization problem, by estimating directly $\psi_t(\cdot)$ from $\psi_{t-1}(\cdot)$ while keeping the estimate in the span of the dictionary atoms. The second framework, called parameter or dual framework, estimates the weighting coefficients in the expansion, namely by updating the unknown vector $\balpha_t$ from $\balpha_{t-1}$. The two frameworks are detailed in the following, starting with the former since its vector-based formulation is straightforward due to connections with adaptive filtering literature. It is worth noting that, while we outline here the most prolific methods for online learning with kernels, we are far from providing a comprehensive survey. 

We denote by $e_t = y_t - \psi_{t-1}(\bx_t)% = y_t - \balpha_{t-1}^\top \dkappa(\bx_t)
$ the prediction error.%, also called the a priori error.% in opposition with the a posteriori error obtained after updating the model, \ie, $y_t - \psi_{t}(\bx_t)$.

\subsection*{Dual framework}

This framework explores the model \eqref{eq:repr_m} written, % it in matrix form with, 
for any $\bx$,% \in \X$,
\begin{equation}\label{eq:repr_M}
        \psi_t(\bx) = \balpha_t^\top \dkappa(\bx),
\end{equation}
where $\balpha_t = [\alpha_{1,t} ~~ \alpha_{2,t} ~~ \cdots ~~ \alpha_{m,t}]\!^\top$ is updated from the previous estimate, \ie, $\balpha_{t-1}$, in the dual space $\R^m$. It is easy to see in \eqref{eq:repr_M} the structure of a a finite-impulse-response filter, the filter input being $\dkappa(\bx)$ and its coefficient vector $\balpha_t$.

%shows a structure similar to a linear adaptive filter.

%the vector $ in the vector-space $\R^m$. {\magenta A MODIFIER Since the model $\psi_t(\cdot)$ in \eqref{eq:repr_m} is $\psi_t(\bx) = \balpha_t\!^\top \dkappa(\bx)$ for any $\bx$, then we have a  with the form $y_t = \balpha_t\!^\top \bx$.}

%\subsection*{Stochastic gradient}

By considering the instantaneous risk 
%\begin{equation*}
	$\tfrac12 | y_t - \balpha^\top \dkappa(\bx_t) |^2 + \epsilon \, \tfrac12 \|\balpha\|^2$, 
%\end{equation*}
where the first term is the quadratic instantaneous error $e_t^2$, we get the stochastic gradient descent rule
\begin{equation}\label{eq:stochastic_gradient}
    \balpha_t = \balpha_{t-1} + \eta_t \big( e_t \, \dkappa(\bx_t)  - \epsilon \, \balpha_{t-1} \big).
\end{equation}
When dealing with the functional regularization $\|\psi(\cdot)\|_\H^2$ as in \eqref{eq:risk_ridge_dual}, this regularization is approximated with $\balpha\!^\top \! \dK \balpha$, which yields the modified version
\begin{equation}\label{eq:stochastic_gradient_K}
    \balpha_t = \balpha_{t-1} + \eta_t \big( e_t \, \dkappa(\bx_t) - \epsilon \, \dK \balpha_{t-1} \big).
\end{equation}
The two rules \eqref{eq:stochastic_gradient} and \eqref{eq:stochastic_gradient_K} reduce to the LMS algorithm when $\epsilon=0$. %Some researchers argue that any regularization induces a bias, while even a small value of $\epsilon$ degrades the performance compared with the LMS algorithm. See for instance \cite{Say03} for the linear case. 
Another algorithm is the NLMS, which provides a scale insensitive version with %the following recursion 
\begin{equation*}%\label{eq:aNLMS}
    \balpha_t = \balpha_{t-1} +  \frac{\eta_t}{\|\dkappa(\bx_t)\|^2 + \epsilon}
                        \, e_t \, \dkappa(\bx_t).
\end{equation*}
See \cite{Hon07.isit} for more details. An extension to an AP algorithm is proposed in \cite{Ric09.tsp}, while a RLS algorithm is presented in \cite{Eng04,Hon07.gretsi.b}. A comprehensive study of adaptive filter algorithms in the dual framework is given in \cite{07.PhD}. See also %This framework is investigated in \cite{Ishida2013Multikernel} for learning multiple dictionaries, in \cite{13.spl.dictionary} for multiple-output learning, as well as in \cite{Yukawa12Multikernel} and \cite{Fan2014}.
\cite{Yukawa12Multikernel,Ishida2013Multikernel,Fan2014}.

\subsection*{Functional framework}

The functional framework considers the definition of the model \eqref{eq:repr_m} in the RKHS, with the form
\begin{equation}\label{eq:repr_H}
        \psi_t(\bx) = \psH{\psi_t(\cdot)}{\kappa(\bx,\cdot)},
\end{equation}
for any $\bx \in \X$. The estimation of $\psi_t(\cdot)$ from the previous estimate $\psi_{t-1}(\cdot)$ is operated in the RKHS $\H$, or more specifically in the span of the available dictionary, \ie, $\psi_t(\cdot) \in \dH \subset \H$.

By considering the instantaneous risk $\tfrac12 | \psi(\bx_t) - y_i |^2 + \epsilon \, \tfrac12 \|\psi(\cdot)\|_\H^2$, the stochastic gradient descent in $\H$ is
\begin{equation*}%\label{eq:stochastic_gradient_1_functional}
    \psi_t(\cdot) = \psi_{t-1}(\cdot) + \eta_t \big( e_t \, \kappa(\bx_t,\cdot)  - \epsilon \, \psi_{t-1}(\cdot) \big).
\end{equation*}
%or equivalently in a ``leaky LMS'' formulation:
%\begin{equation}\label{eq:stochastic_gradient_1_functional}
%    \psi_t(\cdot) = (1 - \eta_t \, \epsilon) \, \psi_{t-1}(\cdot) + \eta_t \, e_t \, \kappa(\bx_t,\cdot).
%\end{equation}
By analogy with the dual framework, other algorithms can also be described such as an LMS, a NLMS, an AP, and a RLS algorithms. See \cite{KernelAdaptiveFiltering} for more details.

Unfortunately, all these formulations assume the finiteness of the training set% in order to use the stochastic gradient descent approach
, as reported in \cite{Kiv04} and \cite{Sma06}. This drawback is due to the fact that the model is fed with a new kernel function at each instant. In order to control this growth and restrict ourselves to the span of the dictionary, we replace\footnote{Besides the approximation with the projection which can be computationally expensive, one may replace the current kernel function with its most collinear atom. This leads to a quantization strategy \cite{QuantizedKLMS}.
} the current $\kappa(\bx_t,\cdot)$ by its projection onto the subspace spanned by the dictionary, namely $\dic{\kappa}_{\bx_t}(\cdot) = \dkappa(\bx_t)\!^\top \dK^{-1} \dkappa(\cdot)$; see Appendix %~\ref{sec:projection}
 for details. This leads to the expression
\begin{equation*}%\label{eq:stochastic_gradient_1_functional_proj}
    \psi_t(\cdot) = (1 - \eta_t \, \epsilon) \, \psi_{t-1}(\cdot) + \eta_t \, e_t \, 
    %\dkappa(\bx_t)\!^\top \dK^{-1} \dkappa(\cdot)
    \dic{\kappa}_{\bx}(\cdot).
\end{equation*}
To implement this formula, one needs to provide an update rule of the parameters, with an expression of the form
\begin{equation*}%\label{eq:stochastic_gradient_1_functional_alpha}
    \balpha_t = (1 - \eta_t \, \epsilon) \, \balpha_{t-1} + \eta_t \, e_t \, 
    \dkappa(\bx_t)\!^\top \dK^{-1}. %\dkappa(\cdot)
\end{equation*}

\section{Online sparsification and sparsity measures}\label{sec:intro_spars}

Independently of %the mathematical foundations and 
the investigated framework, %all the algorithms presented in the previous section suffer from complexity incrementation at each instant, by including $\kappa(\bx_t,\cdot)$ to $\psi_{t-1}(\cdot)$%, as illustrated  as shown in \eqref{eq:kNLMS} and \eqref{eq:kNLMS_Smale}
%. As a consequence, 
online learning algorithms should be coupled with a sparsification scheme. %In an online setting, a novel information $(\bx_t,y_t)$ is available at each instant $t$. 
At each instant, the dictionary is updated if necessary, or it is left unchanged. 
%of the atoms that define the model \eqref{eq:repr_m}
Indeed, the dictionary is augmented whenever the novel kernel function $\kappa(\bx_t,\cdot)$ increases the diversity of the dictionary. There exists several sparsity measures to quantify this diversity, as described in the following.%: the distance (\cf Section~\ref{sec:distance}), the approximation (\cf Section~\ref{sec:approx}), the coherence (\cf Section~\ref{sec:coherence}), and the Babel (\cf Section~\ref{sec:Babel}) measures. 

Before detailing these sparsity measures, we outline the online sparsification scheme. Two cases may arise:
\begin{itemize}
	\item {\bf Case 1}: the dictionary is left unchanged.\\
	This case arises when the novel kernel function $\kappa(\bx_t,\cdot)$ does not contribute significantly to the diversity of the dictionary, and therefore it could be discarded. 
	\item {\bf Case 2}: the kernel function is added to the dictionary.\\
	This case arises when the kernel function $\kappa(\bx_t,\cdot)$ is significantly different from the atoms of the dictionary.
\end{itemize}
One may also use a removal process in the latter case in order to provide a fixed-budget learning \cite{FixedBudgetKRLS10,FixedBudgetKLMS12}, by discarding the atom that has the least contribution to the diversity of the dictionary, as investigated for instance in \cite{NguyenTuong11}.% with the approximation measure. %As opposed to the fixed-budget concept, many theoretical results corroborated by experiments have demonstrated the finiteness of dictionaries resulting from a sparsification process; see for instance \cite{07.PhD,Ric09.tsp}. The work conducted in this paper can be considered in either case, with or without any removal process. 

\subsection{The distance measure}\label{sec:distance}

A simple measure to characterize a sparse dictionary is the least distance between all pairs of its atoms. A dictionary is said to be $\delta$-distant when %the distances between all pairs of its atoms is not smaller than $\delta$. In other words,
\begin{equation}\label{eq:dist}
    \mathop{\min_{i,j=1\cdots m}}_{i \neq j} \min_{\xi} \normH{\kappa(\dic{\bx_i},\cdot) - \xi \, \kappa(\dx_j,\cdot)} \geq \delta,
\end{equation}
where we have included a scaling factor $\xi$. %The square distance in this expression is $\normH{\kappa(\dx_i,\cdot) - \xi \kappa(\dx_j,\cdot)}^2
%	 \!=\! \kappa(\dx_i,\dx_i) -  2 \, \xi \, \kappa(\dx_i,\dx_j)
%	  + \xi^2 \kappa(\dx_j,\dx_j)$. 
%By taking the derivative of the above cost function with respect to $\xi$, and nullifying it, we get the optimal value of the latter, \ie, $\kappa(\dx_i,\dx_j) / \kappa(\dx_j,\dx_j)$, which 
This corresponds to the reconstruction error of projecting $\kappa(\dx_i,\cdot)$ onto $\kappa(\dx_j,\cdot)$, with $\xi = \kappa(\dx_i,\dx_j) / \kappa(\dx_j,\dx_j)$. By substituting this value in \eqref{eq:dist}, we get for any pair $(\dx_i, \dx_j)$:
\begin{equation}\label{eq:dist1}
    \kappa(\dx_i,\dx_i) - \frac{\kappa(\dx_i,\dx_j)^2}{\kappa(\dx_j,\dx_j)} \geq \delta^2.
\end{equation}
%which reduces to $|\kappa(\dx_i,\dx_j)| \leq \sqrt{1-\delta^2}$ for unit-norm atoms.
%For unit-norm atoms, this expression reduces to the condition $|\kappa(\dx_i,\dx_j)| \leq \sqrt{1-\delta^2}$.

A sparsification criterion based on this measure constructs a dictionary with a large distance measure% between its atoms
, thus %discarding candidate kernel functions that are too close to atoms already in the dictionary. Therefore, a 
including the candidate kernel function $\kappa(\bx_t, \cdot)$ in the dictionary if %(and only if)
\begin{equation}\label{eq:dist.crit1}
    \min_{j=1\cdots m}
    \left(
    \kappa(\bx_t,\bx_t) - \frac{\kappa(\bx_t,\dx_j)^2}{\kappa(\dx_j,\dx_j)} 
    \right) \geq \delta^2,
\end{equation}
for some threshold parameter $\delta$. This sparsification criterion is related to the novelty criterion given in \cite{Platt91}, which is the sparsification criterion without the scaling factor followed by a prediction error mechanism.

%, yielding
%%Several variants of the distance measure have been proposed, the most known being the one where the scaling factor is abandoned. This corresponds to the novelty measure: 
%\begin{equation*}%\label{eq:dist}
%    \mathop{\min_{i,j=1\cdots m}}_{i \neq j} \normH{\kappa(\dic{\bx_i},\cdot) -  \kappa(\dx_j,\cdot)} \geq \delta.
%\end{equation*}
%This measure is the main building block of the so-called ``novelty criterion'' introduced by Platt in \cite{Platt91} for the linear kernel, by merging this criterion of closeness between samples with a prediction error mechanism. See also \cite{Rosipal97RAN,Huang05ageneralized}.

\subsection{The approximation measure}\label{sec:approx}

The distance measure defined in \eqref{eq:dist}-\eqref{eq:dist1} relies only on two atoms, that is the closest pair in the dictionary. A more comprehensive analysis of the dictionary composition is the capacity of approximating any atom by a linear combination of the other atoms. A dictionary is designated $\delta$-approximate if the following is satisfied:
\begin{equation}\label{eq:approx}
    \min_{i=1\cdots m} \min_{\xi_1\cdots \xi_m}\normHBig{\kappa(\dx_i,\cdot) - \mathop{\sum_{j=1}^m}_{j \neq i} \xi_j\,\kappa(\dx_j,\cdot)} \geq \delta.
\end{equation}
This corresponds to the reconstruction error of projecting any kernel function $\kappa(\dx_i,\cdot)$ onto the subspace spanned by the other kernel functions. Following the derivation given in Appendix%~\ref{sec:projection}
%, we get 
\begin{equation}\label{eq:approx_temp}
	\bxi = \dK_{\!_{\setminus\!\{\!i\!\}}\!}^{-1} \, \dkappa_{\!_{\setminus\!\{\!i\!\}}\!}(\dx_i),
\end{equation}
where $\dK_{\!_{\setminus\!\{\!i\!\}}\!}$ and $\dkappa_{\!_{\setminus\!\{\!i\!\}}\!}(\dx_i)$ are obtained from $\dK$ and $\dkappa(\dx_i)$, respectively, by removing the entries associated to $\dx_i$. As a consequence, expression \eqref{eq:approx} becomes
\begin{equation}\label{eq:approx1}
	\min_{i=1\cdots m}
	\kappa(\dx_i,\dx_i) - \dkappa_{\!_{\setminus\!\{\!i\!\}}\!}(\dx_i)\!^\top \, \dK_{\!_{\setminus\!\{\!i\!\}}\!}^{-1} \, \dkappa_{\!_{\setminus\!\{\!i\!\}}\!}(\dx_i) \geq \delta^2.
\end{equation}

The (linear) approximation criterion is based on constructing a dictionary with a high approximation measure, as investigated for Gaussian processes in  \cite{Csato01}, for a kernel-based filter in \cite{Eng04} and more recently for kernel principal component analysis in \cite{12.tpami}. %The approximation criterion does not include the candidate kernel function in the dictionary, if it can be sufficiently represented by atoms already in the dictionary. 
The kernel function $\kappa(\bx_t,\cdot)$ is added to the dictionary if% (and only if)
\begin{equation}\label{eq:approx.crit}
    \min_{\xi_1\cdots\xi_m}\normHBig{\kappa(\bx_t,\cdot) - \sum_{j=1}^m \xi_j\,\kappa(\dx_j,\cdot)}^2 \geq \delta^2,
\end{equation}
where $\delta$ is a positive threshold parameter that controls the level of sparseness. This leads to the following condition, written in matrix form $\kappa(\bx_t,\bx_t) - \dkappa(\bx_t)\!^\top \dK^{-1} \dkappa(\bx_t) \geq \delta^2$.
%It is worth noting that $\delta^2 \leq \kappa(\bx_i,\bx_i)$. 

\subsection{The coherence measure} \label{sec:coherence}

The coherence is a fundamental measure to characterize a dictionary in the literature of sparse approximation. It corresponds to the largest correlation between atoms of a given dictionary, or mutually between atoms of two dictionaries. %This concept was introduced at the beginning of the nineties in \cite{Mallat93matchingpursuit} for linear matching pursuit. It has been extensively studied in the beginning of the years 2000's, with the work of Donoho \emph{et al.}, in \cite{Donoho01uncertaintyprinciples} for the union of two orthonormal basis and in \cite{Don03} for basis pursuit with arbitrary dictionaries. 
The coherence measure has been investigated for the analysis of the quality of representing a signal with a dictionary, initially with the work \cite{Gil03a,Tro04}, and more recently in the abundant publications on compressed sensing \cite{Elad2010book}. While most work consider the use of a linear measure, we explore in the following the coherence on kernel functions in order to derive the coherence criterion, as initially proposed in \cite{Hon07.isit,Ric09.tsp}.

A dictionary $\D$ is said $\gamma$-coherent if% (and only if)
\begin{equation}\label{eq:coher}
	\mathop{\max_{i,j=1\cdots m}}_{i \neq j} 
	%\mathrm{coh}(\kappa(\dx_i,\cdot) , \kappa(\dx_j,\cdot))
	%\mathrm{coh}(\dx_i,\dx_j) 
	\frac{|{\kappa(\dx_i,\dx_j)}|} {\sqrt{\kappa(\dx_i,\dx_i) \, \kappa(\dx_j,\dx_j)}}
	\leq \gamma.
\end{equation}
The coherence corresponds to the cosine of the angle between the kernel functions, since the above quotient can be written
\begin{equation*}
%	\mathrm{coh}(\dx_i,\dx_j)
%	\frac{|{\kappa(\dx_i,\dx_j)}|} {\sqrt{\kappa(\dx_i,\dx_i) \, \kappa(\dx_j,\dx_j)}}
%	= 
\frac{|\psH{\kappa(\dx_i,\cdot)}{\kappa(\dx_j,\cdot)}|}
        {\normH{\kappa(\dx_i,\cdot)} \normH{\kappa(\dx_j,\cdot)}}.
\end{equation*}
For unit-norm kernels, %the definition 
\eqref{eq:coher} becomes 
%\begin{equation*}
	$\displaystyle\mathop{\max_{i,j=1\cdots m}}_{i \neq j} 
	|\kappa(\dx_i,\dx_j)| 
	\leq \gamma$.
%\end{equation*}

The coherence criterion constructs a low-coherent dictionary \cite{Hon07.isit,Ric09.tsp}. It includes the candidate kernel function $\kappa(\bx_t,\cdot)$ in the dictionary if %(and only if) 
the coherence of the latter does not exceed a given threshold $\gamma \in \; ] 0 \; ; 1]$, namely% if (and only if)
\begin{equation}\label{eq:coher.crit}
	\max_{j=1 \cdots m}
	\frac{|{\kappa(\bx_t,\dx_j)}|} {\sqrt{\kappa(\bx_t,\bx_t) \, \kappa(\dx_j,\dx_j)}}
	\leq \gamma.
\end{equation}
This condition enforces an upper bound on the cosine of the angle between each pair of kernel functions. The threshold $\gamma$ controls the level of sparseness of the dictionary, where a null value yields an orthogonal basis. This criterion is computationally efficient as given in expression \eqref{eq:coher.crit}, where the denominator reduces to 1 for unit-norm atoms, thus becomes in this case $\displaystyle \max_{j=1\cdots m} |\kappa(\bx_t,\dx_j)| \leq \gamma$.

\subsection{The Babel measure}\label{sec:Babel}

From a norm perspective, the coherence is essentially the $\infty$-norm when dealing with unit-norm atoms. The Babel notion explores such analogy with the norm operator, thus providing a more complete description of the dictionary structure \cite{Gil03b,Tro04}. The Babel is related to the $1$-norm of the Gram matrix, %where $\|\dK\|_1 = \max_i \sum_j |\kappa(\dx_i, \dx_j)|$. 
%{\red In the following, expression are given for the unit-norm kernels; otherwise, substitute $|\kappa(\dx_i, \dx_j)|$ with the coherence defined in \eqref{eq:coher}}.
with the definition %The Babel of a dictionary of kernel functions is defined as 
\begin{equation}\label{eq:babelD}
      \mathrm{Babel}
      = \max_{i=1\cdots m} \mathop{\sum_{j=1}^m}_{j \neq i} |\kappa(\dx_i, \dx_j)|.
\end{equation}
It corresponds to the maximum cumulative correlation between an atom and all the other atoms of the dictionary. It is easy to see that, when dealing with a unit-norm atoms, the coherence of the dictionary cannot exceed its Babel measure.

The Babel criterion is defined as follows. A candidate kernel function $\kappa(\bx_t,\cdot)$ is included in the dictionary if %(and only if) 
%\begin{equation}\label{eq:babel.crit}
      $\sum_{j=1}^m |\kappa(\bx_t, \dx_j)| \leq \gamma$, 
%\end{equation}
for a given positive threshold $\gamma$. This definition can be viewed as an extension of the coherence criterion %\footnote{One can also consider a normalized version of the Babel measure, with 
%$$\max_{j=1 \cdots m}
%	\sum_{j=1}^m \frac{|{\kappa(\bx_t,\dx_j)}|} {\sqrt{\kappa(\bx_t,\bx_t) \, \kappa(\dx_j,\dx_j)}}
%	\leq \gamma.$$
%This definition reduces to \eqref{eq:babel.crit} in the case of unit-norm kernels. To the best of our knowledge, this formulation is not used in the literature. Moreover, it looses the analogy with the matrix-norm notion. For these reasons, we are not using this definition in this paper.} 
in the same sense as the approximation is an extension of the distance criterion. See \cite{Fan2014} for the use of the Babel measure for sparsification.

%%%%%%%%%%%%%%%%%%%%%%%%%%%%%%%
%%%%%%%%%%%%%%%%%%%%%%%%%%%%%%%

\section{An eigenvalue analysis}\label{sec:eigen}

%On the other hand, we examine the eigenvalues associated to the corresponding Gram matrix, by providing lower and upper bounds. These results are fundamental to the analysis of the condition number and the linear independence of the atoms. Furthermore, we conduct an analysis of the uniqueness of the sparsest solution obtained from the dictionary, by exploring the notion of \emph{spark} measure. Moreover, bounds on the eigenvalues are considered in the developments given in Sections~\ref{sec:approx.error} and \ref{sec:compressed_sensing}.

Since the Gram matrix is fundamental in the analysis of the dictionary, we study in the following its eigenvalues, and provide theoretical bounds. These results provide an analysis of the span defined by a sparse dictionary, given in terms of the sparsity measure under scrutiny. Lower bounds are used in the forthcoming linear independence analysis (\cf Section~\ref{sec:lin.indep}), while lower and upper bounds are investigated in the forthcoming study of the condition number (\cf Section~\ref{sec:cond.number}) and in the main results derived in next section (\cf Section~\ref{sec:topology}). Let $\dic{\lambda}_1, \dic{\lambda}_2, \ldots, \dic{\lambda}_m$ be the eigenvalues of the matrix $\dK$, given in non-increasing order, namely $\dic{\lambda}_1 \geq \dic{\lambda}_2 \geq \ldots \geq \dic{\lambda}_m$. %Compared to the eigenvalues of the Gram matrix $\bK$ of all samples, we have 

\subsection{Bounds on the eigenvalues}\label{sec:bounds}

Before proceeding, we bring to mind the well-known Ger\v{s}gorin Discs Theorem \cite[Chapter~6]{Hor12}% and the Cauchy interlacing theorem \cite[Theorem~4.3.17]{Hor12}
, revisited here for the Gram matrix of a sparse dictionary. It is also well known that the trace of a matrix equals the sum of its eigenvalues. We get for unit-norm atoms: $\sum_{j=1}^m \dic{\lambda}_j = \mathrm{Trace}(\dK) =  \sum_{j=1}^m \kappa(\dx_j,\dx_j) = m$, thus $1 \leq \dic{\lambda}_1$ and $\dic{\lambda}_m \leq 1$.

\begin{theorem}[Ger\v{s}gorin Discs Theorem]\label{th:gershgorin}
	Every eigenvalue of an $m$-by-$m$ matrix $\dK$ lies in the union of the $m$ discs, centered on each diagonal entry of $\dK$ with a radius given by the sum of the absolute values of the other $m-1$ entries from the same row. In other words, for each $\dic{\lambda}_i$, there exists at least one $j \in \{1, 2, \ldots, m\}$ such that
	$$|\dic{\lambda}_i - \kappa(\dx_j,\dx_j)|  
		\leq \mathop{\sum_{j=1}^m}_{j \neq i} |\kappa(\dx_i,\dx_j)|.$$
\end{theorem}

%
%Next, we present the Cauchy interlacing theorem, revisited here for the Gram matrix; See \cite{PrincipalSubmatrices,interlacing98} for more details. Let $\dic{\lambda}_{_{\setminus\!\{\!j\!\}}\!1}, \dic{\lambda}_{_{\setminus\!\{\!j\!\}}\!2}, \ldots, \dic{\lambda}_{_{\setminus\!\{\!j\!\}}\!m-1}$ be the eigenvalues, given in non-increasing order, of the matrix$\dK_{\!_{\setminus\!\{\!j\!\}}}$, namely the $(m-1)$-by-$(m-1)$ submatrix of the Gram matrix of the dictionary obtained by removing the $j$-th element from it, \ie, the one associated to $\dx_j$. 
%\begin{lemma}[Cauchy interlacing theorem]\label{th:interlacing}
%	We have the following interlacing property, for any $i=1, 2, \ldots, m-1$:
%	\begin{equation*}
%		\dic{\lambda}_i \geq \dic{\lambda}_{_{\setminus\!\{\!j\!\}}\!i} \geq \dic{\lambda}_{i+1}.
%	\end{equation*}
%\end{lemma}
%This means that, considering the largest and smallest eigenvalues of both matrices, we have respectively: 
%$$\dic{\lambda}_1 \geq \dic{\lambda}_{_{\setminus\!\{\!j\!\}}\!1} \text{ and } \dic{\lambda}_{_{\setminus\!\{\!j\!\}}\!m-1} \geq \dic{\lambda}_m.$$

This theorem is a cornerstone in our study, as described next by providing upper and lower bounds on the eigenvalues of the Gram matrix associated to a sparse dictionary, by investigating its sparsity measure. 

\subsection*{Distance measure}

When the distance measure of a given sparse dictionary is known, namely $\delta$, we have from \eqref{eq:dist}-\eqref{eq:dist1} that any pair $(\dx_i,\dx_j)$ satisfies
$$|\kappa(\dx_i,\dx_j)| \leq \sqrt{\kappa(\dx_j,\dx_j) \, \big( \kappa(\dx_i,\dx_i) - \delta^2 \big)}.$$
Therefore, we have
\begin{align*}%\label{eq:dist.suminnerproduct}
	\sum_{j} |\kappa(\dx_i,\dx_j)|
\nonumber& \leq \sum_j \sqrt{\kappa(\dx_j,\dx_j) \, \big( \kappa(\dx_i,\dx_i) - \delta^2 \big)}
\\	& = \sqrt{ \kappa(\dx_i,\dx_i) - \delta^2 } \sum_j \sqrt{\kappa(\dx_j,\dx_j)}.
\end{align*}
By applying the Ger\v{s}gorin Discs Theorem (Theorem~\ref{th:gershgorin}) with the above relation in mind, we get that, for each eigenvalue $\dic{\lambda}_k$, there exists at least one $i$ such that
\begin{align*}
	|\dic{\lambda}_k - \kappa(\dx_i,\dx_i)|  
	& \leq \mathop{\sum_{j=1}^m}_{j \neq i} |\kappa(\dx_i,\dx_j)|
%\\	& \leq \mathop{\sum_{j=1}^m}_{j \neq i} \sqrt{\kappa(\dx_j,\dx_j) \, \big( \kappa(\dx_i,\dx_i) - \delta^2 \big)}
\\	& \leq \sqrt{ \kappa(\dx_i,\dx_i) - \delta^2 } \mathop{\sum_{j=1}^m}_{j \neq i} \sqrt{\kappa(\dx_j,\dx_j)}.
\end{align*}
By exploring these results, the proof of the following theorem is straightforward.
\begin{theorem}\label{th:eigen.dist}
The eigenvalues of the Gram matrix associated to a $\delta$-distant dictionary are bounded as follows: 
	\begin{align*}
		r^2 - (m-1)R\sqrt{R^2-\delta^2} 
		&\leq 
		\dic{\lambda}_m 
		\leq 
		\cdots 
\\		
		\cdots 
		&\leq
		\dic{\lambda}_1
		\leq 
		R^2 + (m-1)R\sqrt{R^2-\delta^2},
	\end{align*}
	where $r^2 = \inf_{\bx} \kappa(\bx,\bx)$ and $R^2 = \sup_{\bx} \kappa(\bx,\bx)$. For unit-norm atoms, we get
	\begin{equation*}
		1 - (m-1)\sqrt{1-\delta^2} \leq 
		\dic{\lambda}_m \leq 
		\cdots \leq
		\dic{\lambda}_1 \leq 
		1+(m-1)\sqrt{1-\delta^2}.
	\end{equation*}	
	%where the numerator is $\|\dkappa(\dx_j) \|^2$.
\end{theorem}

%The following theorem is straightforward from the definition of the condition number \eqref{eq:cond} and the above theorem.
%\begin{theorem}\label{th:eigen.dist.cond}
%	The condition number of the Gram matrix of a $\delta$-distant dictionary cannot exceed 
%	$$\frac{R^2 + (m-1)R\sqrt{R^2-\delta^2}}{r^2 - (m-1)R\sqrt{R^2-\delta^2}},$$
%namely for unit-norm kernels $\frac{1+(m-1)\sqrt{1-\delta^2}}{1-(m-1)\sqrt{1-\delta^2}}$.
%\end{theorem}

\subsection*{Approximation measure}

Presented here for completeness, the following theorem is essential due to Honeine in \cite{12.tpami}.
\begin{theorem}\label{th:eigen.approx}
The eigenvalues of the Gram matrix associated to a $\delta$-approximate dictionary are bounded as follows:
	\begin{equation*}
		\delta^2 \leq 
		\dic{\lambda}_m \leq 
		\cdots \leq
		\dic{\lambda}_1 \leq 
		2 R^2 - \delta^2,
	\end{equation*}
	where $R^2 = \sup_{\bx} \kappa(\bx,\bx)$.
\end{theorem}

\begin{proof}
By injecting  \eqref{eq:approx_temp} in \eqref{eq:approx1}, we get $\min_{\bxi} \kappa(\dx_i,\dx_i) - \dkappa_{\!_{\setminus\!\{\!i\!\}}\!}(\dx_i)\!^\top \bxi \geq \delta^2$, for any $i=1,2, \ldots, m$, or equivalently
\begin{equation*}
	\max_{\bxi} \mathop{\sum_{j=1}^m}_{j \neq i} \xi_j \, \kappa(\dx_i,\dx_j) 
	\leq \kappa(\dx_i,\dx_i) - \delta^2.
\end{equation*}
Considering the special case (which could be far from the optimum) of $\xi_j=\mathrm{sign}(\kappa(\dx_i,\dx_j))$, we get
\begin{equation*}%\label{eq:approx1.c}
	 \mathop{\sum_{j=1}^m}_{j \neq i} |\kappa(\dx_i,\dx_j)|
	\leq \kappa(\dx_i,\dx_i) - \delta^2.
\end{equation*}
The proof of the theorem follows from the Ger\v{s}gorin Discs Theorem (Theorem~\ref{th:gershgorin}), namely for any eigenvalue $\dic{\lambda}_k$, there exists an $i$ such that
\begin{equation*}
	|\dic{\lambda}_k - \kappa(\dx_i,\dx_i)|  
		\leq \mathop{\sum_{j=1}^m}_{j \neq i} |\kappa(\dx_i,\dx_j)|
		\leq \kappa(\dx_i,\dx_i) - \delta^2.
		%\qedhere
%		\leq R^2 - \delta^2.
\end{equation*}
\vskip -.5cm
\end{proof}

\subsection*{Coherence measure}

When measuring the sparsity of the dictionary with the coherence measure, we have the following theorem. Only the lower bound has been previously investigated in the literature when dealing with unit-norm atoms; see \cite{Hon07.isit}.
\begin{theorem}\label{th:eigen.coher}
The eigenvalues of the Gram matrix associated to a $\gamma$-coherent dictionary of $m$ atoms are bounded as follows:
	\begin{equation*}
		r^2-(m-1) \gamma R^2 \leq 
		\dic{\lambda}_m \leq 
		\cdots \leq
		\dic{\lambda}_1 \leq 
		R^2+(m-1) \gamma R^2,
	\end{equation*}
	where $R^2 = \sup_{\bx} \kappa(\bx,\bx)$ and $r^2 = \inf_{\bx} \kappa(\bx,\bx)$. For unit-norm atoms, we get $$1-(m-1) \, \gamma \leq \dic{\lambda}_m \leq \cdots \leq \dic{\lambda}_1 \leq 1+(m-1) \, \gamma.$$
\end{theorem}
\begin{proof}
A $\gamma$-coherent dictionary satisfies
\begin{equation*}
	\mathop{\max_{j=1\cdots m}}_{j \neq i} 
	%\mathrm{coh}(\kappa(\dx_i,\cdot) , \kappa(\dx_j,\cdot))
	%\mathrm{coh}(\dx_i,\dx_j) 
	\frac{|{\kappa(\dx_i,\dx_j)}|} {\sqrt{\kappa(\dx_i,\dx_i) \, \kappa(\dx_j,\dx_j)}}
	\leq \gamma,
\end{equation*}
for any $i=1,2,\ldots,m$, which yields
\begin{align*}
	\mathop{\max_{j=1\cdots m}}_{j \neq i} |\kappa(\dx_i,\dx_j)|
	&\leq \gamma 
	\mathop{\max_{j=1\cdots m}}_{j \neq i} \sqrt{\kappa(\dx_i,\dx_i) \, \kappa(\dx_j,\dx_j)}
\\	&= \gamma \sqrt{\kappa(\dx_i,\dx_i)}
	\mathop{\max_{j=1\cdots m}}_{j \neq i} \sqrt{\kappa(\dx_j,\dx_j)}
\\	&\leq \gamma R \sqrt{\kappa(\dx_i,\dx_i)}.
\end{align*}
Finally, the proof results from applying the Ger\v{s}gorin Discs Theorem (Theorem~\ref{th:gershgorin}), since%, for any eigenvalue $\dic{\lambda}_k$, there exists an $i$ such as
\begin{align*}
	%|\dic{\lambda}_k - \kappa(\dx_i,\dx_i)|  
	%	\leq 
	\mathop{\sum_{j=1}^m}_{j \neq i} |\kappa(\dx_i,\dx_j)|
		&\leq (m-1) \mathop{\max_{j=1\cdots m}}_{j \neq i} |\kappa(\dx_i,\dx_j)| 
\\		&\leq (m-1) \gamma R \sqrt{\kappa(\dx_i,\dx_i)}
\\		&\leq (m-1) \gamma R^2.
\end{align*}
%This concludes the proof.
\vskip -.25cm
\end{proof}

%The proof of the following theorem is straightforward, and shows the impact of the threshold on the conditioning of the system.
%\begin{theorem}\label{th:eigen.coher.cond}
%	The condition number of the Gram matrix of a $\gamma$-coherent dictionary is upper-bounded by $$\frac{R^2+(m-1) \gamma R^2}{r^2-(m-1) \gamma R^2};$$
%and for unit-norm kernels $\frac{1+(m-1)\gamma}{1-(m-1)\gamma}$.
%\end{theorem}

\subsection*{Babel measure}

When dealing with the Babel measure as a sparsity measure, the eigenvalues of the Gram matrix associated to the dictionary are bounded as given in the following theorem.
\begin{theorem}\label{th:eigen.babel}
The eigenvalues of the Gram matrix associated to a $\gamma$-Babel dictionary are bounded as follows:
	\begin{equation*}
		r^2-\gamma \leq 
		\dic{\lambda}_m \leq 
		\cdots \leq
		\dic{\lambda}_1 \leq 
		R^2+\gamma,
	\end{equation*}
	where $R^2 = \sup_{\bx} \kappa(\bx,\bx)$ and $r^2 = \inf_{\bx} \kappa(\bx,\bx)$. For unit-norm atoms, we get 
$1-\gamma \leq \dic{\lambda}_m \leq \cdots \leq \dic{\lambda}_1 \leq 1+\gamma$.
\end{theorem}

\begin{proof}
The proof follows from the Ger\v{s}gorin Discs Theorem (Theorem~\ref{th:gershgorin}) since, for any eigenvalue $\dic{\lambda}_k$, there exists an $i\in \{1, 2, \ldots, m\}$ with
\begin{equation*}
	|\dic{\lambda}_k - \kappa(\dx_i,\dx_i)|  
		\leq 
	\mathop{\sum_{j=1}^m}_{j \neq i} |\kappa(\dx_i,\dx_j)|
		\leq \gamma.
\end{equation*}
\vskip -.5cm
\end{proof}

%The proof of the following theorem is straightforward, and shows the impact of the threshold in the Babel criterion on the conditioning of the system.
%\begin{theorem}\label{th:eigen.babel.cond}
%	The condition number of the Gram matrix of a $\gamma$-Babel dictionary cannot exceed $$\frac{R^2+\gamma}{r^2-\gamma}.$$
%\end{theorem}

\subsection{Linear independence}\label{sec:lin.indep}

It is relevant to construct a dictionary with linearly independent atoms, a condition that allows to represent any feature of $\D_\H$ in a unique linear way. For a dictionary of $m$ kernel functions, the atoms are linearly independent if the following is satisfied: any linear combination $\sum_{j=1}^m \xi_j \, \kappa(\dx_j,\cdot)$ is the zero element if and only if all the weighting coefficients $\xi_j$ are null.

It is trivial that a dictionary with an nonzero approximation measure has linear independent atoms, since we have 
\begin{align*}%\label{eq:approx}
\! \normHBig{\sum_{j=1}^m \xi_j \, \kappa(\dx_j,\cdot)}
	\!\!\! & = \normHBig{\xi_i \kappa(\dx_i,\cdot) - \mathop{\sum_{j=1}^m}_{j \neq i} \xi_j\,\kappa(\dx_j,\cdot)} ~ (\text{for any } i)
\\	& = |\xi_i| \, \normHBig{\kappa(\dx_i,\cdot) - \mathop{\sum_{j=1}^m}_{j \neq i} \frac{\xi_j}{\xi_i} \, \kappa(\dx_j,\cdot)}
%\\	& = |\xi_i| \normHBig{\kappa(\dx_i,\cdot) - \mathop{\sum_{j=1}^m}_{j \neq i} \xi_j \, \kappa(\dx_j,\cdot)}
\\	& \geq |\xi_i| \min_{\xi_1\cdots\xi_m} \normHBig{\kappa(\dx_i,\cdot) - \mathop{\sum_{j=1}^m}_{j \neq i} \xi_j \, \kappa(\dx_j,\cdot)}
\\	& \geq |\xi_i| \, \delta,
\end{align*}
for any decomposition, \ie, $i \in \{1,2, \ldots, m\}$. Thus, the linear combination is the zero element only when all coefficients $\xi_i$ are null or when the threshold $\delta$ is null.
%where the last equality is due to a variable change with $\xi_j = \frac{\xi_j}{\xi_i}$

In the following, we show that all the sparsity measures provide sufficient conditions for linear independence of the dictionary's atoms. To this end, we investigate the duality between linear independence and the non singularity of the associated Gram matrix, which is essentially considered in \cite{Gil03a} for the coherence of a linear dictionary with unit-norm atoms and extended in \cite{Ric09.tsp} for kernel-based dictionaries. Indeed, we have
\begin{equation*}
	\normHBig{\sum_{j=1}^m \xi_j \, \kappa(\dx_j,\cdot)}^2=\bxi\!^\top \! \dK \bxi \geq \dic{\lambda}_m\|\bxi\|^2,
\end{equation*}
where the Courant-Fischer Minimax Theorem is used \cite[Theorem~8.1.2]{matrix13}. As a consequence, we prove the linear independence of a atoms by providing a lower bound on the eigenvalues of the associated Gram matrix. The following theorem summarizes this property for different sparsity measures.
\begin{theorem}[Linear independence]\label{th:lin_ind}
	A sufficient condition for the linear independence of the $m$ atoms is:
\begin{itemize}
	\item $(m-1)R\sqrt{R^2-\delta^2} < r^2$ for a $\delta$-distant dictionary.
	\item $\delta >0$ for a $\delta$-approximate dictionary.
	\item $(m-1) \gamma R^2 < r^2$ for a $\gamma$-coherent dictionary.
	\item $\gamma < r^2$ for a $\gamma$-Babel dictionary.
\end{itemize}
\end{theorem}
These results generalize the bounds given for only unit-norm atoms, in \cite{Hon07.isit} for the coherence measure with $(m-1) \gamma < 1$ and in \cite{Ric09.tsp} for the Babel measure with $\gamma < 1$.

\subsection{Condition number}\label{sec:cond.number}

The condition number of a matrix $\dK$, for a given matrix norm, is defined by $\mathrm{cond}(\dK) = \|\dK\| \|\dK^{-1}\|$, which reduces for the $\ell_2$-norm to:
\begin{equation}\label{eq:cond}
	\mathrm{cond}(\dK) = \frac{|\dic{\lambda}_1|}{|\dic{\lambda}_m|}.
\end{equation}
It is an important measure of the sensitivity, with respect to variations within the matrix $\dK$, of the resolution of a problem of the form $\dK \balpha = \by$, $\balpha$ being the unknown. It gives a bound on how inaccurate the solution $\balpha$ will be after approximation. When its value is small, \ie, close de 1, the solution is robust to perturbations, as opposed to large values that lead to ill-conditioned problems, if not even ill-posed. 

For instant, consider a gradient descent procedure to solve the linear system $\dK \balpha = \by$. It is shown in \cite{Lue89} that the error reduction at each iteration is bounded by an upper bound that is proportional to the condition number of the matrix $\dK$. The condition number has been studied more recently in kernel-based machine learning; see for instant \cite{Kur05}. Next, we provide an upper bound on the condition number, in terms of the sparsity measure of the dictionary. The proof of the following theorem is straightforward from the definition of the condition number \eqref{eq:cond} and the aforementioned theorems on lower and upper bounds on the eigenvalues.

\begin{theorem}[Condition number]\label{th:cond}
	The condition number of the Gram matrix associated to a sparse dictionary is upper-bounded by:
\begin{itemize}
\item $\displaystyle \frac{R^2 + (m-1)R\sqrt{R^2-\delta^2}}{r^2 - (m-1)R\sqrt{R^2-\delta^2}}$ for a $\delta$-distant dictionary.
\item $\displaystyle \frac{2 R^2}{\delta^2}-1$ for a $\delta$-approximate dictionary.
\item $\displaystyle \frac{R^2+(m-1) \gamma R^2}{r^2-(m-1) \gamma R^2}$ for a $\gamma$-coherent dictionary.
\item $\displaystyle \frac{R^2+\gamma}{r^2-\gamma}$ for a $\gamma$-Babel dictionary.
\end{itemize}
\end{theorem}
The case of unit-norm atoms is obtained from the relation $r = R = 1$, which yields for instance the upper bound $\frac{1+(m-1)\gamma}{1-(m-1)\gamma}$ for a $\gamma$-coherent dictionary. These results demonstrate how the choice of the threshold value in the sparsification criterion impacts on the conditioning of the system, towards a well-posed optimization problem.

%%%%%%%%%%%%%%%%%%%%%%%%%%%%%%%
%%%%%%%%%%%%%%%%%%%%%%%%%%%%%%%
\section{Connecting the dictionary's induced feature space and the dual space}\label{sec:topology}

In this section, we show that both feature subspace and the dual space are intimately related in their topologies, when the feature subspace is spanned by the atoms from a sparse dictionary. To this end, we show in Section~\ref{sec:isometry} that the pairwise distances in both spaces are almost preserved. This quasi-isometry property associated to a given sparse dictionary is quantified in terms of each of the sparsity measures presented in Section~\ref{sec:intro_spars}, namely the distance, approximation, coherence, and Babel measures. These results on the isometry are extended in Section~\ref{sec:isometry_ip} to the issue of preserving the pairwise inner-products in both spaces. All these results establish the structural-preserving map that connects both spaces, namely the map $\Theta_\D$ defined as follows
\begin{eqnarray*}
	\Theta_\D \colon 
		\R^m & \longmapsto & \, \dH  \subset  \H \\
		\balpha ~~ &\longrightarrow & \psi(\cdot) = \balpha\!^\top \dkappa(\cdot)
\end{eqnarray*}
%the function that maps each $\balpha$ in $\R^m$ to $\sum_{j=1}^m \alpha_j \, \kappa(\dx_j,\cdot)$ in $\dH$ for a given dictionary $\D$.

It is worth noting that these results require that the atoms of the dictionary are linear independent, since this condition guarantees that any feature $\psi(\cdot)$ of $\dH$ can be uniquely represented by atoms of the dictionary. See Section~\ref{sec:lin.indep} and in particular Theorem~\ref{th:lin_ind} which provides weak conditions in terms of the sparsity measure of the dictionary.

\subsection{Isometry property}\label{sec:isometry}

Without limiting ourselves to online learning by comparing $\psi_t(\cdot)$ with $\psi_{t-1}(\cdot)$, we consider here any two features from the feature space $\dH$, denoted $\psi'(\cdot) = \sum_{j=1}^m \alpha_{j}' \, \kappa(\dx_j,\cdot)$ and $\psi''(\cdot) = \sum_{j=1}^m \alpha_{j}'' \, \kappa(\dx_j,\cdot)$. Their representations in the dual space $\R^m$ are denoted $\balpha'$ and $\balpha''$, respectively. There exists an isometry between these two spaces if the distance between any pair of features corresponds to the distance between their parameter vectors, namely $\normH{\psi'(\cdot) - \psi''(\cdot)} =\norm{\balpha' - \balpha''}$. While the isometry property is too restrictive, we relax it with the following definition of quasi-isometry, by showing that the quotient of these two distances is close to unity. We denote $\psi(\cdot) = \psi'(\cdot) - \psi''(\cdot)$, then its parameter vector is $\balpha = \balpha' - \balpha''$.%, we get $\normH{\psi(\cdot)} \approx \norm{\balpha}$.%, \ie, the vector of entries $\alpha_{j} = \alpha_{j,t} - \alpha_{j,t'}$, for $j=1,2, \ldots, m$.

\begin{definition}[Quasi-isometry]\label{th:isometry_def}
Given a dictionary of kernel functions $\{\kappa(\dx_1,\cdot),\kappa(\dx_2,\cdot), \ldots, \kappa(\dx_m,\cdot)\}$, and $\dH$ the space spanned by its atoms, we say that the spaces $\R^m$ and $\dH$ are quasi-isometric if there exists an isometry constant $\nu$ (the smallest number) such that, for any vector $\balpha$ of entries $\alpha_j$, the feature $\psi(\cdot)=\balpha\!^\top \dkappa(\cdot)$ satisfies  
\begin{equation}\label{eq:isometry}
1-\nu \leq \frac{\normH{\psi(\cdot)}^2}{\|\balpha\|_2^2} \leq 1+\nu.
\end{equation}
\end{definition}
This means that the map $\Theta_\D \colon \balpha \to \balpha\!^\top \dkappa(\cdot)$ approximately preserves the distances in both spaces $\R^m$ and $\dH$. It is easy to see that a dictionary with an isometry constant $\nu=0$ provides a ``total'' isometry between these spaces. 

In the following, we show that the quasi-isometry property is satisfied for sparse dictionaries, by relying on the investigated sparsity measure. Before generalizing with Theorem~\ref{th:isometry}, we restrict ourselves in Theorem~\ref{th:isometry_unit} to the case of unit-norm atoms, which is often sufficient in most work in the literature of sparse approximation, \eg, when using the Gaussian kernel.

\begin{theorem}[Isometry property --unit-norm atoms--]\label{th:isometry_unit}
A dictionary of unit-norm atoms has an isometry constant $\nu$ defined as follows:
\begin{itemize}
\item $\nu = (m-1) \sqrt{1 - \delta^2}$ for a $\delta$-distant dictionary.
\item $\nu = 1 - \delta^2$ for a $\delta$-approximate dictionary.
\item $\nu = (m-1) \gamma$ for a $\gamma$-coherent dictionary.
\item $\nu = \gamma$ for a $\gamma$-Babel dictionary.
\end{itemize}
\end{theorem}

\begin{proof}
For any $\psi(\cdot)$ with its parameter vector $\balpha$ we have $\normH{\psi(\cdot)}^2 = \normH{\sum_{j=1}^m \alpha_j \, \kappa(\dx_j,\cdot)}^2 = \balpha\!^\top \! \dK \balpha$, then the quotient in \eqref{eq:isometry} is the Rayleigh-Ritz quotient of the Gram matrix $\dK$. By applying the Courant-Fischer Minimax Theorem, we get
\begin{equation*}%\label{eq:isometry}
 \dic{\lambda}_m \leq \frac{\normH{\psi(\cdot)}^2}{\|\balpha\|_2^2} \leq \dic{\lambda}_1,
\end{equation*}
where $\dic{\lambda}_m$ and $\dic{\lambda}_1$ and the smallest and largest eigenvalues of the matrix $\dK$. We can easily identify from \eqref{eq:isometry} the following pair of inequalities:
\begin{equation*}%\label{eq:isometry}
1-\nu \leq\dic{\lambda}_m  \qquad \text{and} \qquad \dic{\lambda}_1 \leq 1+\nu.
\end{equation*}
By exploring the results derived in Section~\ref{sec:eigen}, we can identify the isometry constants of the dictionary in terms of its distance, approximation, coherence and Babel measures. %Note that these relations are valid only when the linear independence condition is satisfied, as studied in Section~\ref{sec:lin.indep}. 
Besides the approximation measure, all these expressions are straightforward from Theorems~\ref{th:eigen.dist}, \ref{th:eigen.coher}, \ref{th:eigen.babel}, thanks to the bounds on the eigenvalues that are symmetric about 1. Even in the asymmetric bounds of the approximation measure as given in Theorem~\ref{th:eigen.approx}, that is $\delta^2 \leq \dic{\lambda}_m \leq \cdots \leq\dic{\lambda}_1 \leq 2 - \delta^2$, one can easily identify the expression of the isometry constant $\nu = 1 - \delta^2$.
\end{proof}

When dealing with non-unit-norm atoms, expressions are a bit more difficult to derive, due to the asymmetry of the bounds on the eigenvalues, as shown by the following theorem.
\begin{theorem}[Isometry property]\label{th:isometry}
A dictionary %of kernel functions $\{\kappa(\dx_1,\cdot), \kappa(\dx_2,\cdot), \ldots, \kappa(\dx_m,\cdot)\}$ 
has an isometry constant $\nu$ defined as follows:
\begin{itemize}
\item $\displaystyle\nu = \frac{R^2-r^2 + 2 (k-1) R \sqrt{R^2-\delta^2}}{R^2+r^2}$ for a $\delta$-distant dictionary.
\item $\displaystyle\nu = 1 - \frac{\delta^2}{R^2}$ for a $\delta$-approximate dictionary.
\item $\displaystyle\nu = \frac{R^2-r^2 + 2 (k-1) \gamma R^2}{R^2+r^2}$ for a $\gamma$-coherent dictionary.
\item $\displaystyle\nu = \frac{R^2-r^2 + 2 \gamma}{R^2+r^2}$ for a $\gamma$-Babel dictionary.
\end{itemize}
In these expressions, $R^2 = \sup_{\bx} \kappa(\bx,\bx)$ and $r^2 = \inf_{\bx} \kappa(\bx,\bx)$.
\end{theorem}

\begin{proof}
Consider the general asymmetric bounds
\begin{equation*}
l_k\leq\dic{\lambda}_m \leq \frac{\normH{\psi(\cdot)}^2}{\|\balpha\|_2^2} \leq \dic{\lambda}_1 \leq u_k,
\end{equation*}
for some lower bound $l_k$ and upper bound $u_k$, such that $0 < l_k \leq u_k < \infty$. In order to get bounds that are symmetric about~1, as in Definition~\ref{th:isometry_def}, we divide each term by ${({u_k+l_k})/{2}}$. This yields the isometry constant $\nu = (u_k-l_k)/(u_k+l_k)$ for the rescaled atoms of the dictionary, where each atom is divided by $\sqrt{({u_k+l_k})/{2}}$. Finally, the proof of the theorem follows the same steps as in the proof of Theorem~\ref{th:isometry_unit}.
\end{proof}
It is easy to see that Theorem~\ref{th:isometry_unit} is a special case of this theorem when dealing with unit-norm atoms, \ie, $R=r=1$.

\subsection{Preserving inner products}\label{sec:isometry_ip}

Theorems~\ref{th:isometry_unit} and \ref{th:isometry} show that a sparse dictionary provides a quasi-isometry, with respect to the distances, between the dual space and the subspace spanned by its atoms. In the following, we show that this property of quasi-isometry extends to inner products. It is worth noting that, when dealing with a ``total'' isometry, the isometry with respect to inner products extends naturally to the isometry with respect to distances, and vice versa\footnote{For any linear operator $\bA$ from an inner product space to another inner product space, there exists an equivalence between $\langle \bA \bu , \bA \bv \rangle = \langle \bu , \bv \rangle$ for any $(\bu , \bv)$ and $\|\bA \bu\| = \|\bu\|$ for any $\bu$.This equivalence is less obvious when dealing with quasi-isometry.}. This is not the case when using the quasi-isometry definition. We aim to bridge this gap in the following.% when dealing with inner products.

\begin{definition}[Quasi-isometry w.r.t. inner products]\label{th:isometry_ip}
Given a dictionary of kernel functions $\{\kappa(\dx_1,\cdot),\kappa(\dx_2,\cdot), \ldots, \kappa(\dx_m,\cdot)\}$, and $\dH$ the space spanned by its atoms, we say that the spaces $\R^m$ and $\dH$ are quasi-isometric with respect to inner products if there exists an isometry constant $\nu$ (the smallest number) such that, for any pair of vectors $(\balpha',\balpha'')$, we have 
\begin{equation}\label{eq:isometry_ip}
\frac{\left| \Big\langle{\sum_{j=1}^m \alpha_{j}' \, \kappa(\dx_j,\cdot)},{\sum_{j=1}^m \alpha_{j}'' \, \kappa(\dx_j,\cdot)} \Big\rangle_\H - \balpha'\!^\top \balpha'' 
\right|}{\|\balpha'\|_2 \, \|\balpha''\|_2}
 \leq \nu.
\end{equation}
\end{definition}

It is easy to see that the ``total'' isometry with respect to inner products corresponds to $\nu=0$ in \eqref{eq:isometry_ip}. This expression becomes $\big\langle{\sum_{j=1}^m \alpha_{j}' \, \kappa(\dx_j,\cdot)},{\sum_{j=1}^m \alpha_{j}'' \, \kappa(\dx_j,\cdot)} \big\rangle_\H = \balpha'\!^\top \balpha''$, and as a consequence the condition \eqref{eq:isometry} is satisfied as a special case where $\balpha' = \balpha''$.

In the general case, the quotient in \eqref{eq:isometry_ip} can be written as 
\begin{equation*}
\frac{\big| \balpha'\!^\top \! \dK\balpha'' - \balpha'\!^\top \balpha''
\big|}{\|\balpha'\|_2 \, \|\balpha''\|_2}
=
\frac{\big| \balpha'\!^\top \! (\dK - \bI) \balpha'' \big|}{\|\balpha'\|_2 \, \|\balpha''\|_2},
\end{equation*}
and therefore the inequality \eqref{eq:isometry_ip} becomes 
\begin{equation}\label{eq:isometry_ip2}
- \nu \leq \frac{ \balpha'\!^\top \! (\dK - \bI) \, \balpha'' }{\|\balpha'\|_2 \, \|\balpha''\|_2}
 \leq \nu.
\end{equation}
To tackle this expression, several issues need to be addressed. First of all, the above quotient needs to be connected to the Rayleigh-Ritz quotient of the matrix $\dK - \bI$, in order to apply the Courant-Fischer Minimax Theorem. Indeed, this theorem can be also applied to study a quotient of the form
\begin{equation*}
	\frac{ \bu\!^\top \! \bA \, \bv }{\|\bu\|_2 \, \|\bv\|_2},
\end{equation*}
for any pair $(\bu,\bv)$, as shown in \cite[Theorem~8.6.1]{matrix13}; see also \cite[Theorem~3]{Xiang06minimax} for a detailed proof. As a consequence, the quotient in \eqref{eq:isometry_ip2} is bounded by the extreme eigenvalues of the matrix $\dK - \bI$. Second, it is easy to see that both matrices $\dK$ and $\dK - \bI$ share the same eigenvectors, while for any eigenvalue $\dic{\lambda}_j$ of $\dK$ corresponds the eigenvalue $\dic{\lambda}_j-1$ of $\dK - \bI$. Indeed, any eigenpair $(\dic{\bv},\dic{\lambda}_j)$ of $\dK$ satisfies 
%\begin{equation*}
	$\big( \dK - \bI \big) \dic{\bv} = \dK \dic{\bv} - \bI \dic{\bv} = \dic{\lambda}_j \dic{\bv} - \bI \dic{\bv} = \big(\dic{\lambda}_j -1 \big) \dic{\bv}$,  
%\end{equation*}
therefore $(\dic{\bv},\dic{\lambda}_j-1)$ is an eigenpair of the matrix $\dK - \bI$.

As a consequence, one can take advantage of bounds on the eigenvalues from Theorems~\ref{th:eigen.dist}, \ref{th:eigen.approx}, \ref{th:eigen.coher} and \ref{th:eigen.babel} to provide expressions for the isometry constant w.r.t. inner products, as detailed in Theorems~\ref{th:isometry_unit} %for unit-norm atoms and more generally in Theorem
and~\ref{th:isometry}.

\section{Final remarks}\label{sec:final_remarks}

This paper provided a framework, based on an eigenvalue analysis, to study sparsity measures and sparsification criteria. We proposed a unified study for the well-conditioning of the optimization problem and for the condition on the uniqueness of the solution. We established a quasi-isometry between the dual space and the dictionary's induced feature space, thus connecting the functional to the dual frameworks and illustrating the %These results showed the 
impact of the sparsity measures on the topologies. As for future work, we are extending this framework to include new insights on sparse dictionary analysis.

%
%analyze sparse dictionaries in terms of several sparsity measures, by studying the corresponding eigenvalues 
%showed that all sparsification criteria 
%
%\bigskip
%
%{\red 
%Remarque : faire le lien avec la RIP ... et conclusion : une topologie similaire dans les deux espaces: $\dH$ et celui des paramètres $\balpha$\\
%
%
%To conclude, consider the online learning problem with a filter-based sparsification criterion. Recall that one should choose a large threshold value $\delta$ for the distance and approximation criteria, while the threshold value $\gamma$ should be small enough for the coherence and Babel criteria. From the above relations, we see that the isometry constant monotonically decreases with $\delta$ and monotonically increases with $\gamma$. Therefore, each filter-based sparsification criterion provides an upper bound on the isometry constant.
%
%
%
%\subsection{On an optimal value of the sparsification threshold parameter}

%{The distance criterion as a rough approximation criterion}
%
%The approximation criterion as rough coherence and Babel criteria
%
%

%------------------------------------------------------------------------------------
% Appendix
%------------------------------------------------------------------------------------
\appendix

\section{Projection in a RKHS}\label{sec:projection}

The projection of any kernel function $\kappa(\bx,\cdot)$ onto the subspace spanned by a dictionary of kernel functions $\kappa(\dx_j,\cdot)$, for $j=1,2, \ldots, m$, takes the form 
$$
	\dic{\kappa}_{\bx}(\cdot) = \sum_{j=1}^m \xi_j \, \kappa(\dx_j,\cdot),
$$
or equivalently 
$\dic{\kappa}_{\bx}(\cdot) = \bxi\!^\top \! \dkappa(\cdot)$, where $\bxi$ is obtained by minimizing the quadratic reconstruction error 
\begin{equation}\label{eq:proj_quad_reconst_error}
	\normH{\kappa(\bx,\cdot) - \bxi\!^\top \! \dkappa(\cdot)}^2.
\end{equation}
The expansion of this norm is given by $\kappa(\bx,\bx) - 2 \, \bxi\!^\top \! \dkappa(\bx) + \bxi\!^\top \! \dK \bxi$. By taking its derivative with respect to $\bxi$ and nullifying it, we get 
$$\dK \bxi = \dkappa(\bx).$$
Therefore, the projection is given by 
$$\dic{\kappa}_{\bx}(\cdot) = \dkappa(\bx)\!^\top \dK^{-1} \dkappa(\cdot).$$
The quadratic reconstruction error of such approximation is obtained by substituting this expression into \eqref{eq:proj_quad_reconst_error}, yielding 
$$\kappa(\bx,\bx) - \dkappa(\bx)\!^\top \dK^{-1} \dkappa(\bx).$$

\bigskip\bigskip

%\section*{Acknowledgment} \addcontentsline{toc}{section}{Acknowledgment}
%The author would like to thank C\'edric Richard for the helpful discussions.

%------------------------------------------------------------------------------------
% Bibliographie
%------------------------------------------------------------------------------------

%%\scriptsize
\bibliography{biblio_ph,bibdesk_Paul}
\bibliographystyle{ieeetr}

\begin{biography}[{}]%\includegraphics[width=1in,height=1.25in,clip,keepaspectratio]{Honeine_p}}]
{Paul Honeine} (M'07) was born in Beirut, Lebanon, on October 2, 1977. He received the Dipl.-Ing. degree in mechanical engineering in 2002 and the M.Sc. degree in industrial control in 2003, both from the Faculty of Engineering, the Lebanese University, Lebanon. In 2007, he received the Ph.D. degree in Systems Optimisation and Security from the University of Technology of Troyes, France, and was a Postdoctoral Research associate with the Systems Modeling and Dependability Laboratory, from 2007 to 2008. Since September 2008, he has been an assistant Professor at the University of Technology of Troyes, France. His research interests include nonstationary signal analysis and classification, nonlinear and statistical signal processing, sparse representations, machine learning. Of particular interest are applications to (wireless) sensor networks, biomedical signal processing, hyperspectral imagery and nonlinear adaptive system identification. He is the co-author (with C. Richard) of the 2009 Best Paper Award at the IEEE Workshop on Machine Learning for Signal Processing. Over the past 5 years, he has published more than 100 peer-reviewed papers. 
\end{biography}

%Specialist in machine learning and data mining with over 10 years of experience. Exceptional skills in nonlinear and statistical signal processing, with particular interest in (wireless) sensor networks, biomedical signal processing, hyperspectral imagery and nonlinear adaptive system identification. Outstanding skills rewarded over the past five years, including best paper awards and more than 100 peer-reviewed papers.

\end{document}